\documentclass[11pt]{article}
\pdfoutput=1
\usepackage{enumerate}
\usepackage{pdfsync}
\usepackage[OT1]{fontenc}
\usepackage{smile}
\usepackage[colorlinks,
            linkcolor=red,
            anchorcolor=blue,
            citecolor=blue
            ]{hyperref}
\usepackage{color, colortbl}
\usepackage{fullpage}
\usepackage[protrusion=true,expansion=true]{microtype}
\usepackage{pbox}
\usepackage{setspace}
\usepackage{tabularx}
\usepackage{float}
\usepackage{wrapfig,lipsum}
\usepackage{enumitem}
\usepackage{microtype}
\usepackage{graphicx}
\usepackage{subfigure}
\usepackage{enumerate}
\usepackage{enumitem}
\usepackage{pgfplots}
\usetikzlibrary{arrows,shapes,snakes,automata,backgrounds,petri}
\usepackage{booktabs} 

\usepackage{smile}

\usepackage{hyperref}
\usepackage{url}
\usepackage{fvextra}
\usepackage{mdframed}
\usepackage{xspace}
\usepackage{subcaption}
\usepackage{wrapfig}
\usepackage{placeins}
\usepackage{tabularx}
\usepackage{fancyvrb}

\newcommand{\alg}{\text{Alg}}
\newcommand{\alge}{{\text{Alg}^E}}
\newcommand{\algc}{{\text{Alg}^C}}
\newcommand{\algtheta}{{\text{Alg}_{\hat\btheta}}}

\newcommand{\npre}{{n_{\text{pre}}}}

\newcommand{\ereal}{{\epsilon_{\text{real}}}}
\newcommand{\epre}{{\epsilon_{\text{pretrain}}}}

\newcommand{\decoder}{{\mathtt{Dec}}}

\newcommand{\disth}{{\text{D}_{\text{H}}}}
\newcommand{\disttv}{{\text{D}_{\text{TV}}}}
\newcommand{\ps}{{\hat \PP^s}}
\newcommand{\phat}{{\hat \PP}}

\newcommand{\event}{{\cE}}
\newcommand{\name}{{\textit{\textsc{CoPS}}\xspace}}
\newcommand{\offlinename}{{\textsc{OfflineCoPS}\xspace}}
\newcommand{\onlinename}{{\textsc{OnlineCoPS}\xspace}}

\newcommand{\llm}{{\text{LLM}}}

\usepackage{enumitem}

\allowdisplaybreaks
\usepackage{xcolor}

\usepackage{enumitem}

\makeatletter
\newcommand*{\rom}[1]{\expandafter\@slowromancap\romannumeral #1@}
\makeatother
\title{\huge CoPS: Empowering LLM Agents with Provable Cross-Task Experience Sharing}

\author
{
	Chen Yang\thanks{Equal Contribution} \thanks{Department of Computer Science, Indiana University Bloomington, IN 47408, USA; e-mail: {\tt cya2@iu.edu}}~~
    Chenyang Zhao$^*$\thanks{Department of Computer Science, University of California, Los Angeles, CA 90095, USA; e-mail: {\tt  zhaochenyang@cs.ucla.edu}}~~~
    Quanquan Gu\thanks{Co-Corresponding Authors} \thanks{Department of Computer Science, University of California, Los Angeles, CA 90095, USA; e-mail: {\tt  qgu@cs.ucla.edu}}~~~
	Dongruo Zhou$^\S$\thanks{Department of Computer Science, Indiana University Bloomington, IN 47408, USA; e-mail: {\tt dz13@iu.edu}} 
}

\begin{document}
\date{}
\maketitle

\begin{abstract} Sequential reasoning in agent systems has been significantly advanced by large language models (LLMs), yet existing approaches face limitations. Reflection-driven reasoning relies solely on knowledge in pretrained models, limiting performance in novel scenarios, while experience-assisted reasoning often depends on external experiences and lacks clear principles for selecting representative experiences. We address these limitations by proposing \name{} (\textbf{\underline{C}}r\textbf{\underline{o}}ss-Task Ex\textbf{\underline{p}}erience \textbf{\underline{S}}haring), a generalizable algorithm that enhances sequential reasoning by cross-task experience sharing and selection. In detail, \name{} leverages agents' experiences on previous tasks, selecting distribution-matched experiences via a provable pessimism-based strategy to maximize utility while minimizing risks from distribution shifts. Extensive experimental results on benchmarks like Alfworld, Webshop, and HotPotQA demonstrate that \name{} consistently outperforms state-of-the-art baselines, with superior sample efficiency suitable for resource-constrained scenarios. Theoretically, we show that the performance of our algorithm depends on both the quality of the pretrained LLM and the matching between the agent's task-dependent trial distribution and that generated by the LLM. Our work bridges the gap between existing sequential reasoning paradigms and validates the effectiveness of leveraging cross-task experiences, shedding light on the potential to improve agents' generalization and adaptability across diverse tasks. Our codes are available at \href{https://github.com/uclaml/COPS}{https://github.com/uclaml/COPS}.
\end{abstract}

\section{Introduction}
Burgeoning agent systems driven by advanced large language models (LLMs, \citep{DBLP:conf/naacl/DevlinCLT19,DBLP:conf/nips/BrownMRSKDNSSAA20,DBLP:journals/corr/abs-2303-08774,hu2024minicpm}) have demonstrated remarkable capabilities in solving complex tasks through sequential reasoning \citep{qin2024large, hao2023reasoning, huang2024understanding, chen2024internet, chen2024agentverse, li2023camel}. These agent systems employ two typical sequential reasoning paradigms: reflection-driven reasoning and experience-assisted reasoning. Reflection-driven reasoning leverages a model's internal capabilities through methods such as reflection \citep{shinn2024reflexion}, long-term rollouts \citep{zhou2023language}, or chain-of-thought (CoT) reasoning \citep{wei2022chain}. While this approach capitalizes on the knowledge within the pre-trained model, it faces notable limitations. Specifically, relying solely on existing knowledge in the pre-trained model to generate rationales restricts the model's performance when encountering novel scenarios. Moreover, there is an increased risk of hallucinations, where internal reasoning may lead to plausible but incorrect responses \citep{huang2023survey}. These challenges highlight the need for integrating external experiences to enhance the agent's sequential reasoning capabilities.

In contrast, experience-assisted sequential reasoning utilizes retrieval-based methods that enable the agent to interact with a memory bank of experiences, allowing the model to overcome knowledge cutoffs, personalize responses, and reduce hallucinations. However, these experiences are often manually curated or sourced from expert models \citep{raparthy2023generalization}, which is resource-intensive and poses scalability issues. Additionally, experience-assisted reasoning often lacks clear principles for selecting representative examples \citep{kagaya2024rap}, potentially underutilizing the value of past experiences. These limitations bring us to a critical research question:

\begin{figure}[t!]
    \centering
        \includegraphics[width=\textwidth]{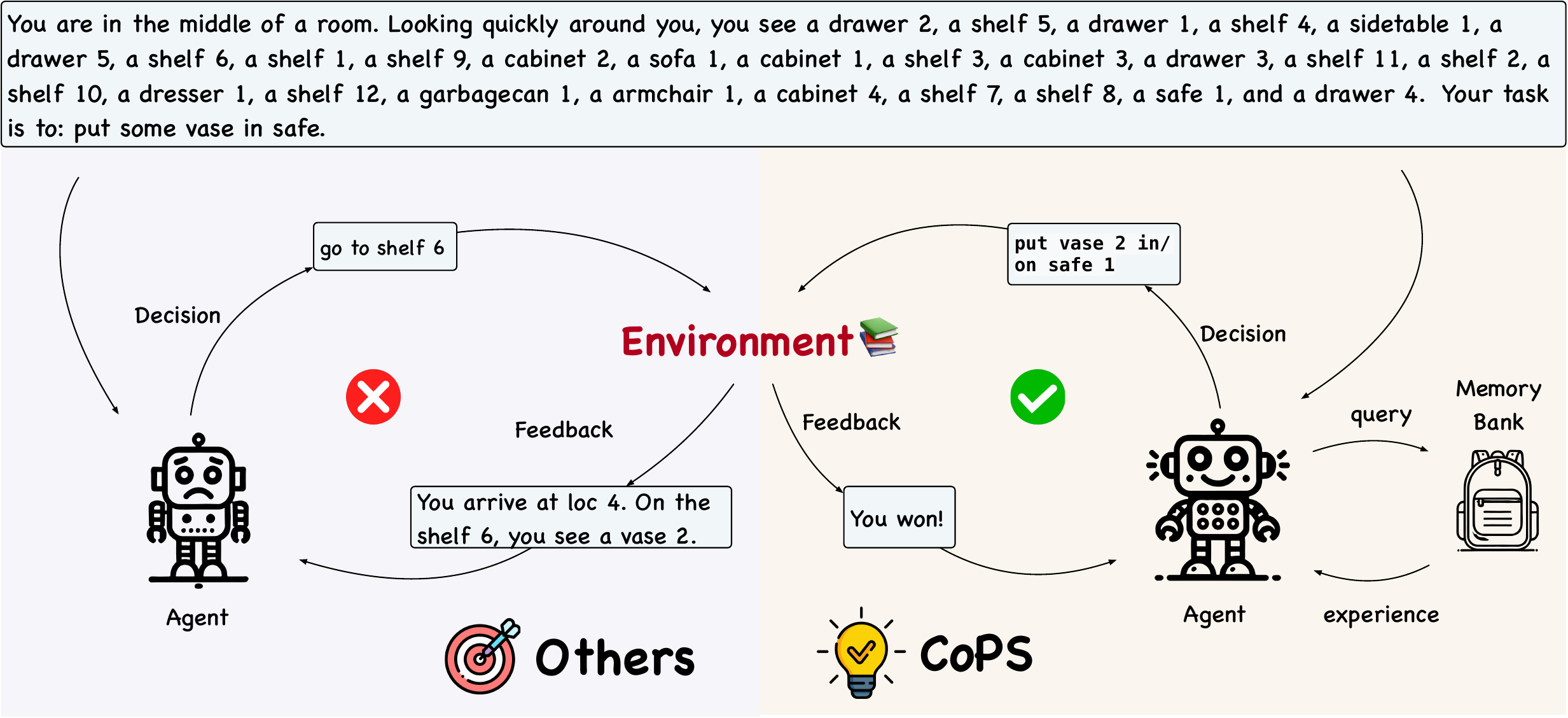} 
    \vskip\baselineskip
    \vspace{1em}  
    \caption{A brief illustration of \name, which fully leverages agents' cross-task experiences to enhance sequential reasoning by sharing and selecting distribution-matched experiences from previous task trajectories.}
\end{figure}

\begin{center} \textit{Can agent systems enhance sequential reasoning by sharing and selecting cross-task experiences?} \end{center}

To address this question, we propose \name{} (\textbf{\underline{C}}r\textbf{\underline{o}}ss-Task Ex\textbf{\underline{p}}erience \textbf{\underline{S}}haring), a theoretically grounded algorithm that empowers agent systems through cross-task experience sharing and selection. \name{} demonstrates its generalizability by working effectively in both settings: utilizing fully external experiences in the \emph{offline} setting and leveraging completely self-derived experiences in the \emph{online} setting. By utilizing representative cross-task experiences, \name{} enables agents to improve performance on new, complex sequential reasoning tasks. Our key contributions are summarized as follows:

\begin{itemize}[leftmargin=*] 

\item We introduce \name{}, a method that fully leverages agents' cross-task experiences to enhance sequential reasoning by selecting distribution-matched experiences from previous task trajectories. Central to our approach is a theoretically grounded experience selection strategy based on the pessimism principle, which aims to maximize the utility of successful, representative experiences while minimizing risks associated with distribution shifts from out-of-distribution samples. Notably, \name{} is agnostic to the agent's base model, task type, experience sources, and implementation framework, making it easy-to-use and generalizable across various settings.

\item Experimentally, we validate \name{} on key benchmarks such as Alfworld \citep{shridhar2020alfworld}, Webshop \citep{yao2022webshop}, and HotPotQA \citep{yang2018hotpotqa}. \name{} consistently outperforms state-of-the-art experience-assisted reasoning approaches like RAP \citep{kagaya2024rap} and reflection-driven reasoning methods, like Reflexion \citep{shinn2024reflexion} and LATS \citep{zhou2023language}. Moreover, \name{} demonstrates superior sample efficiency compared to resource-intensive methods like LATS, making it highly suitable for resource-constrained scenarios. These results showcase \name{}'s effectiveness in practical applications.

\item Theoretically, we show that in both offline and online settings, the performance of our pessimism-based algorithm depends on both the quality of the pre-trained LLM and the matching between the cross-task experience distribution decided by the trials selected by the agent, and a task-dependent experience distribution denoted by the LLM. Our findings shed light on general strategies for designing efficient experience sharing and selction algorithms and offer a comprehensive understanding of \name{}'s effectiveness across different scenarios.

\end{itemize}

\noindent\textbf{Notations}
We denote by $[n]$ the set $\{1,\dots, n\}$. For two positive sequences $\{a_n\}$ and $\{b_n\}$ with $n=1,2,\dots$, we write $a_n=O(b_n)$ if there exists an absolute constant $C>0$ such that $a_n\leq Cb_n$ holds for all $n\ge 1$ and write $a_n=\Omega(b_n)$ if there exists an absolute constant $C>0$ such that $a_n\geq Cb_n$ holds for all $n\ge 1$. We use $\tilde O(\cdot)$ to further hide the polylogarithmic factors. We use $(x_i)_{i=1}^n$ to denote sequence $(x_1, ..., x_n)$, and we use $\{x_i\}_{i=1}^n$ to denote the set $\{x_1, ...,x_n\}$.
We use $\disth(p,q) = \sqrt{1/2\cdot \int (\sqrt{p} - \sqrt{q})^2}$ to denote the Hellinger distance. We use $\disttv(p,q) = 1/2\cdot\int|p-q|$ to denote the Total variation distance. We use $\chi^2(p,q) = \int p^2/q -1$ to denote the chi-square distance. For two sentences $a$ and $b$, we use $a|b$ to denote the sentence formed by concatenating $a$ and $b$. We use $\Delta(\cD)$ to denote the set of distributions over $\cD$.

\section{Related Work}

\subsection{LLM-powered Agents}
In recent years, there has been a significant surge in research focused on LLM-powered agents \citep{chen2024internet, chen2024agentverse, chan2023chateval}. React \citep{yao2022react} laid the foundation for much of the subsequent work on LLM agents, particularly those based on in-context learning (ICL). The most relevant studies to \name{} include \citet{shinn2024reflexion, kagaya2024rap, zhou2023language, raparthy2023generalization}. In \citet{kagaya2024rap}, a retrieval process for selecting in-context demonstrations was proposed. However, their approach depends on frequent embedding queries during the planning stage, leading to inefficiency issues even in smaller LLM settings. Additionally, RAP manually splits the agent's planning trajectory into multiple stages for each trial, with benchmark-specific tailoring, significantly increases implementation complexity and raises scalability concerns. \citet{zhou2023language} introduced a Tree-of-Thought (ToT) approach \citep{yao2024tree}, incorporating backpropagation and a valuation process. However, their approach demonstrated poor sample efficiency, making it less suited for real-world agent settings where opportunities for trial and error are limited. Similarly, \citet{liu2023reason} integrated value-based search into a theoretical framework, but faced similar challenges with sample efficiency. \citet{feng2024agile} explored fine-tuning for specific LLM agent tasks, achieving good performance but with high computational costs. Lastly, \citet{raparthy2023generalization} utilized high-quality experiences as ICL demonstrations for sequential reasoning. Although achieving remarkable performance, these experiences are introduced from external RL systems, which is resource-intensive and poses scalability issues.

\subsection{In-context Demonstrations Selection}
The selection of demonstrations for ICL has been widely studied. \citet{wang2024large} approached in-context demonstration selection from a Bayesian perspective, explicitly constructing a latent variable for the selection process. However, their analysis did not account for the pre-trained knowledge distribution, and their results were primarily empirical. \citet{yan2023understanding} investigated the impact of repetition in in-context demonstrations, conducting controlled experiments to assess how repetitions in pre-trained knowledge influence results. \citet{scarlatos2023reticl} developed a reinforcement learning framework to select in-context examples, while \citet{voronov2024mind} examined the impact of prompt formatting on in-context learning performance. Additionally, \citet{shum2023automatic} introduced an automatic CoT augmentation and selection method for ICL example datasets. \citet{hu2024unveiling} analyzed the scaling of in-context demonstrations from a theoretical standpoint, deriving general statistical bounds while accounting for pre-training errors. However, their focus was primarily on CoT in general ICL settings, not on the specific challenges faced by LLM agents interacting with environments and requiring feedback for optimization.

\subsection{Theory of Agents}
Several works have advanced the theoretical understanding of LLM agents. 
\citet{he2024words} explored the statistical theory of LLM agents through the lens of Bayesian aggregated imitation learning. \citet{lin2023transformers} provided a theoretical analysis of transformers within the context of in-context reinforcement learning. \citet{wang2024understanding} examined the training and generalization of transformers for sequential reasoning, drawing parallels between transformer behavior and online learning algorithms. \citet{sumers2023cognitive} offered a cognitive perspective on LLM agents, while \citet{park2024llm} investigated the regret of LLM agents in sequential reasoning tasks, contributing both theoretical and empirical insights that inform \name{}'s development.

\section{Methodology}\label{sec:propmethod}

\subsection{Preliminary}

We consider a sequential decision-making scenario, consisting of a task space $\mathtt{M}$, a state space $\mathtt{S}$, and an action space $\mathtt{A}$. The state $s \in \mathtt{S}$ is defined as a descriptive sentence representing the history of the current task. For example: “You are in the middle of a room. Please find a path to reach the apple.” The action $a \in \mathtt{A}$ is a solution to the task, such as: “Move right. The apple is on the table.” The agent interacts with the environment through trials. At the beginning of each trial, a task $\Mb$ is randomly drawn from the task space, $\Mb \sim \mathbb{P}^{\mathtt{M}}$. The agent then observes an initial state $s_1$, sampled from the initial state distribution, $s_1 \sim \mathbb{P}^{\Mb}_0$. At each step $h$, the agent makes a decision $a_h$ based on the current state $s_h$, and the next state is updated as $s_{h+1} = s_h | a_h$. The agent either successfully completes the task or continues generating actions until reaching the maximum number of interactions $H$ between the agent and the environment. We define an \emph{experience} $\tau$ as a complete trial, i.e., $\tau = s_h$, where $h \leq H$ is the final step of the current trial. The reward $r(s_h)$ denotes how effectively the experience solves the task, with $0 \leq r(s_h) \leq 1$.

In this work, we assume access to a large language model (LLM) to assist in decision-making. We represent the LLM as $\llm(a|\cdot)$, a conditional distribution of actions given the input sequence.

\begin{algorithm}[t!]
\caption{$\name$: \textbf{\underline{C}}r\textbf{\underline{o}}ss-Task Ex\textbf{\underline{p}}erience \textbf{\underline{S}}haring}
\label{alg:proposed_method}
\begin{algorithmic}[1]
\REQUIRE Language model $\llm(\cdot|\cdot)$, memory bank $\cD = \{\tau_1,\dots, \tau_n\}$, decoder $\decoder$, distance metric $d$, memory size $k$, maximum sequence length $H$.
\STATE Receive initial state $s_1$, receive state-sampled experience $\tau^{s_1}$ through decoder $\tau^{s_1}\sim \decoder(\cdot|s_1)$.\label{line:decoder}
\STATE Set the probability $\hat p \in \Delta(\cD)$ as in \eqref{transform}, which approximately maximizes the following:
\begin{align}
    \hat p = \argmax_{p\in \Delta(\cD)} \EE_{\tau\sim p} [r(\tau) - d(\tau, \tau^{s_1})].\label{help:122}
\end{align}
\STATE Repeatedly retrieve trials $\tau^1, \dots, \tau^k \sim \hat p$.  \label{line:sample}
\STATE Concate $\tau^1, \dots, \tau^k$ into one trajectory $\cT = \tau^1 |\dots |\tau^k$, set $h \leftarrow 1$.
\WHILE{NOT SUCCESS and $h<H$}
\STATE Obtain action $a_h\sim \llm(\cdot|\cT, s_h)$, set $s_{h+1} \leftarrow s_h |a_h$, $h\leftarrow h+1$.
\ENDWHILE
\end{algorithmic}
\end{algorithm}

\subsection{Proposed Method}

We introduce our proposed method, \name{}, based on distribution matching. \name{} operates on a trial-wise basis, making it suitable for both the \emph{offline setting}, where the agent has access to an external static dataset containing experiences, and the \emph{online setting}, where the agent gathers experiences through interactions with the environment. Suppose our agent is at the start of a trial with an initial state $s_1 \sim \mathbb{P}_0^{\Mb}$. We introduce the key components of \name{} as follows.

\noindent \textbf{Memory Bank} The agent has access to a memory bank $\mathcal{D}$ containing experiences, either from a pre-collected dataset (offline) or from previous experiences (online). We do not impose restrictions on $\mathcal{D}$, meaning that experiences in $\mathcal{D}$ exhibit great diversity. Specifically, an experience $\tau \in \mathcal{D}$ may correspond to different tasks $\Mb$ or to varying solution strategies for the same task. Our goal is to develop a strategy for retrieving experiences from $\mathcal{D}$ that assist in decision-making for the current task.

\noindent \textbf{Cross-Task Experience Sharing} \name{} utilizes an external module called the \emph{decoder}, denoted as $\decoder$ in Line \ref{line:decoder}. In general, the decoder outputs a task-dependent distribution of experiences conditioned on the initial state $s_1$, reflecting how the LLM would solve the task $\Mb$ associated with $s_1$ without explicit instructions. With the decoder's help, the agent's goal is to find a probability distribution $\hat{p}$ over all experiences in $\mathcal{D}$ that satisfies:
\begin{align}
    \hat p = \argmax_{p\in \Delta(\cD)} \EE_{\tau\sim p} [r(\tau)] - d(p, \decoder(\cdot|s_1)),\label{help:123}
\end{align}
where $d$ is a metric over distributions. Intuitively, \eqref{help:123} is similar to the \emph{pessimism principle}, commonly used in offline RL literature \citep{jin2021pessimism}. The goal of $\hat{p}$ is to maximize the expected reward while keeping the distribution close to the one decoded by $\decoder$. Importantly, $\hat{p}$ supports the cross-task setting, as it does not restrict its support to experiences from the same task as $s_1$. For a given in-context memory size $k$, \name{} repeatedly samples experiences $\tau^1, \dots, \tau^k$ from $\hat{p}$, as shown in Line \ref{line:sample}.

\noindent \textbf{Execution Planning} Let $\mathcal{T} = \tau^1 | \dots | \tau^k$ represent the \emph{experience collection} containing $\tau^1, \dots, \tau^k$. Starting from the initial state $s_1$, the agent executes actions step-by-step, where each action $a_h$ is drawn from the LLM’s distribution, conditioned on both the experience collection and the current state:
\begin{equation} 
a_h \sim \llm(\cdot|\cT, s_h). \notag
\end{equation} 
In the online setting, after completing a trial, the agent updates the memory bank $\mathcal{D}$ by adding the new experience for future use.

\noindent\textbf{Implementation Details} Here we discuss several implementation details for \name{}. \emph{First}, in practice, directly computing the distance $d(p, \decoder(\cdot|s_1))$ between distributions in \eqref{help:123} is computationally intractable. Therefore, we use an empirical approximation to translate the distance between distributions into the distance between experiences drawn from those distributions, as shown in \eqref{help:122}. \emph{Second}, we specify the choice of $\decoder$. The decoder outputs an experience $\tau^{s_1}$ from $\mathcal{D}$ that starts with the same initial state $s_1$. If multiple such experiences exist, we select the most recent one. This $\tau^{s_1}$ naturally reflects the behavior of the LLM for solving the task starting from $s_1$ without intervention. \emph{Third}, we discuss how to approximately solve \eqref{help:122} since enumerating all possible distributions in $\Delta(\mathcal{D})$ is computationally inefficient. Specifically, we define the distance function $d$ and approximately solve $\hat{p}$ as follows:
\begin{equation}
    d(\tau, \tau'):= c\cdot\text{cos}(e(\tau), e(\tau')),\ \hat p(\tau) \propto r(\tau)\cdot\exp(-d(\tau, \tau^{s_1})),\label{transform}
\end{equation}
where $c \geq 0$ is a hyperparameter, ``cos" denotes the cosine function, and $e$ is an embedding function that maps a language sentence to a high-dimensional Euclidean space. In practice, we use $e$ as a language embedding model (e.g., gte-Qwen2 7b \citep{li2023towards}). This approach favors selecting successful experiences from $\mathcal{D}$ with probabilities proportional to the inverse distance from the current initial state $s_1$. The hyperparameter $c$ in \eqref{transform} controls the influence of relative distances: when $c=0$, the method uniformly samples successful experiences from $\mathcal{D}$, and as $c \to \infty$, it deterministically selects the experience closest to $\tau^{s_1}$.

\section{Experiments}
In this section, we present our experimental study evaluating the practical performance of \name{} on real-world LLMs, specifically the Llama 3.1 models \citep{dubey2024llama}. Our results show that \name{} achieves state-of-the-art (SOTA) performance in both task success rate and sample efficiency, surpassing existing baselines to the best of our knowledge. A detailed description of our prompt formulation is provided in Appendix \ref{prompt}. Notably, \name{} is both simple to implement and generalizable across different environments: for each trial, the selected experiences are straightforwardly added to the prompts, requiring no manual modifications.

This prompting strategy offers two distinct advantages: first, it significantly boosts sequential reasoning performance by incorporating cross-task experiences, outperforming reflection-driven methods like Reflexion. Second, the prompts across trials share a substantial prefix, which maximizes the effectiveness of prefix-caching mechanisms in modern LLM-serving systems \citep{zheng2023efficiently}, leading to significant efficiency improvements over RAP \citep{kagaya2024rap}.

\noindent\textbf{Benchmarks} We evaluate our algorithms on three representative benchmarks: \textbf{Alfworld} \citep{shridhar2020alfworld}, \textbf{Webshop} \citep{yao2022webshop}, and \textbf{HotPotQA} \citep{yang2018hotpotqa}. In these benchmarks, agents strive to solve problems in limited number of trials, enabling cross-trial and cross-task experience sharing. In Alfworld, agents are provided with a specific task description within a simulated household environment, interacting through predefined actions and receiving feedback in the form of textual descriptions. In Webshop, the agent must locate a product that matches user specifications from a catalog of over one million items, interacting with the HTML page and search engine while receiving limited product information per trial. In HotPotQA, the agent answers complex questions requiring specific knowledge, using Wikipedia to retrieve relevant articles. In all benchmarks, the reward function $r(\tau)$ is defined as $1$ if the agent successfully completes the task and $0$ otherwise.

\noindent\textbf{LLM Selection} We conduct our entire experiment with the widely-used Llama 3.1 series of models \citep{dubey2024llama}, in consideration of their superior benchmark performance and the sustainability of open-weight LLM ecosystems\footnotemark[1]. Specifically, our experiments are conducted with Llama 3.1 8b Instruct and Llama 3.1 70b Instruct on NVIDIA A6000 and A100 GPUs. We use gte-Qwen2 7b Instruct \citep{li2023towards} as our embedding model. We use SGLang \citep{zheng2023efficiently} as our LLM-serving engine for its SOTA serving performance and prefix-caching mechanism.

\footnotetext[1]{Note that these are a bunch of representative and popular open source models \citep{wang2024helpsteer2preferencecomplementingratingspreferences,he2023teacherlm,xiao2024configurable, hu2024minicpm} can be considered. We choose the latest model at our experiment time due to resource limitations.}

\noindent\textbf{Baselines} We compare \name{} with three representative baselines: \textbf{Reflexion} \citep{shinn2024reflexion}, \textbf{RAP} \citep{kagaya2024rap}, and \textbf{LATS} \citep{zhou2023language}. In Reflexion, the agent try to solve the task in each environment over multiple trials until it succeeds. After each failed attempt, the LLM agent reflects on its unsuccessful trajectory and saves this reflection in its memory. For each subsequent trial, the agent is provided with up to three recent reflections from the same task. In RAP, at each stage within a trial, the agent is presented with the top-$k$ search results of trajectory fragments as in-context demonstrations. In LATS, the agent utilizes a tree-structured search to explore multiple reasoning and action rationales at each trial. When it encounters failed rationales, the agent generates reflections on its mistakes and integrates these insights into its decision-making process for future trials.

\subsection{Results and Analysis}\label{knnresults}
In this section, we demonstrate that \name{} outperforms all baselines across all benchmarks and model sizes, considering both sample efficiency and task success rate. Detailed performance illustrations over multiple trials are presented in Figure \ref{fig:allresult}. Our hyperparameter details are provided in Table \ref{tab:benchmark_hyperparams_updated} in Appendix \ref{exp:details}.

\begin{wraptable}{r}{0.5\textwidth}
\centering
\caption{Performance comparison of Reflexion, RAP, and \name{} on Alfworld benchmark using Llama3.1 8b and 70b models.}

\begin{tabular}{c|c|c}
\toprule
\multirow{2}{*}{Algorithm} & \multicolumn{2}{c}{\textbf{Performance}} \\
                           & \textbf{Llama3.1 8b} & \textbf{Llama3.1 70b} \\
\midrule
Reflexion \footnotemark[2]    & 86 & 94 \\
RAP          & 70                     & 93 \\
\textbf{\name} & \textbf{94}          & \textbf{100} \\
\bottomrule
\end{tabular}

\label{tab:alfworld_baseline}
\end{wraptable}

\footnotetext[2]{The original codebase of Reflexion struggles to perform on most tasks with the smaller Llama3.1 8b model. This is primarily because the model tends to repeat the same action, leading to task failure. To mitigate this, we introduced a resampling mechanism to enhance Reflexion performance, which activates when the model begins to repeat actions. This modification significantly improved Reflexion's performance.}

\noindent\textbf{Alfworld Benchmark} Table \ref{tab:webshop_baseline_rate} and Figures \ref{alf:8b}, \ref{alf:70b} illustrate the comparison between \name{}, Reflexion, and RAP on the Alfworld benchmark. The values represent the success rate after 10 trials across 134 tasks. When using the smaller Llama 3.1 8b model, \name{} reaches a success rate of 94\%, significantly surpassing both Reflexion (86\%) and RAP (70\%). This result is particularly noteworthy because Reflexion requires the much larger Llama 3.1 70b model to achieve similar performance, highlighting superior effectiveness of \name{}. This demonstrates \name{}'s ability to achieve state-of-the-art performance even with limited computational resources and less capable models, offering a clear advantage over other algorithms. Furthermore, when scaling to the larger Llama 3.1 70b model, \name{} achieves a perfect success rate of 100\%. These results emphasize that \name{} scales effectively, consistently outperforming the baselines across model sizes. Although RAP also leverages an in-context demonstrations retrieval mechanism, it lacks an effective experiences selection algorithm, thus noticeably underperforms \name. Additionally, it is important to note that RAP manually splits the agent's planning trajectory into multiple stages for each trial, and these split methods are specific to each benchmark and must be manually tailored. This significantly increases implementation complexity and introduces scalability issues. In contrast, \name{} efficiently reuses successful experiences by directly placing them in the prompts, without requiring benchmark-specific modifications, making it a more practical and flexible solution. As a result, \name{} not only surpasses the baselines in performance but also offers out-of-the-box usability by eliminating the need for manual intervention.

\begin{wraptable}{r}{0.5\textwidth}
\centering
\caption{Performance comparison of Reflexion, RAP, LATS, and \name{} on Webshop benchmark using Llama3.1 8b and 70b models.}

\begin{tabular}{c|c|c}
\toprule
\multirow{2}{*}{Algorithm} & \multicolumn{2}{c}{\textbf{Performance}} \\
                           & \textbf{Llama3.1 8b} & \textbf{Llama3.1 70b} \\
\midrule
Reflexion    & 30 & 30 \\
RAP          & 42 & 42 \\
LATS         & 24 & 32 \\
\textbf{\name} & \textbf{50} & \textbf{56} \\
\bottomrule
\end{tabular}

\label{tab:webshop_baseline_rate}
\end{wraptable}

\noindent\textbf{Webshop Benchmark\footnotemark[3]} Table \ref{tab:webshop_baseline_rate} and Figures \ref{web:8b}, \ref{web:70b} compare the performance of \name{} with all baseline algorithms on the Webshop benchmark, measured in terms of success rate. The values indicate the success rate over 50 products, with each algorithm evaluated through 10 trials per product. For the smaller Llama 3.1 8b model, \name{} achieves a success rate of 50\%, outperforming the next best competitor, RAP, by a substantial absolute improvement of 8\%. When scaling to the larger Llama 3.1 70b model, the performance gain of \name{} becomes even more pronounced, with a success rate of 56\%. This marks a 14\% absolute improvement over RAP.

\footnotetext[3]{We observed that scaling up the model sizes for Reflexion and RAP on the Webshop benchmark did not result in significant improvements. This observation aligns with the original findings of Reflexion \citep[Appendix B.1]{shinn2024reflexion} and RAP \citep[Table 2]{kagaya2024rap}, which suggest that these models tend to converge on local minima that require highly creative strategies to overcome.}

To ensure a fair comparison across the baselines, we modified the LATS baseline by reducing the width of the search tree and limiting the number of trajectory iterations. This adjustment ensures that the running time spent on each baseline is approximately equal. Even with these changes, LATS still exhibits significantly lower sample efficiency. Specifically, the total number of tokens generated by Llama 3.1 8b in LATS (1,555,365 tokens) is nearly five times greater than that in \name{} (314,336 tokens). Further details can be found in Table \ref{exp:tokens}. This discrepancy in token usage highlights the inefficiency of current search-tree-based algorithms. In contrast, \name{} demonstrates much better efficiency and performance under the same inference constraints.

\begin{table}[t!]
\centering
\caption{Token generation count for each of the \textbf{Webshop} experiments. It's worth notice that for each models the LATS token generation count is at least 5 times to \name.}
\begin{tabular}{c|c|c|c|c}
\toprule
\textbf{Algorithm} & \textbf{Reflexion}  & \textbf{RAP} & \textbf{LATS} & \textbf{\name}\\
\midrule
\textbf{Llama3.1 8b} & 159131 & 107504 & 1555365& 314336\\
\textbf{Llama3.1 70b} & 125406 &109245 & 1058752 &113849 \\
\bottomrule
\end{tabular}
\label{exp:tokens}
\end{table}

\begin{wraptable}{r}{0.5\textwidth}
\centering
\caption{Performance comparison of Reflexion, LATS, and \name{} on HotPotQA benchmark using Llama3.1 8b and 70b models.}
\begin{tabular}{c|c|c}
\toprule
\multirow{2}{*}{Algorithm} & \multicolumn{2}{c}{\textbf{Performance}} \\
                           & \textbf{Llama3.1 8b} & \textbf{Llama3.1 70b} \\
\midrule
Reflexion    & 56 & 61 \\
LATS         & 55 & 64 \\
\textbf{\name} & \textbf{63} & \textbf{65} \\
\bottomrule
\end{tabular}
\vspace{0.0cm}

\label{tab:hotpotqa_baseline}
\end{wraptable}
\noindent\textbf{HotPotQA Benchmark} Table \ref{tab:hotpotqa_baseline} and Figures \ref{hot:8b}, \ref{hot:70b} illustrate the comparison between \name{}, Reflexion, and LATS on the HotPotQA benchmark, conducted on 100 question-answering (QA) tasks. The values in the table represent the success rates, with each algorithm being tested over 10 trials. As evidenced by the results, \name{} consistently achieves superior performance relative to both Reflexion and LATS across all model sizes. The advantage of \name{} is particularly evident when using the smaller Llama 3.1 8b model, where \name{} achieves a success rate of 63\%, outperforming Reflexion and LATS by substantial absolute improvements of 7\% and 8\%, respectively. Moreover, even when scaled up to the larger Llama 3.1 70b model, \name{} continues to gain stronger performance. In this setting, \name{} reaches a success rate of 65\%, surpassing Reflexion by 4\% and LATS by 1\%. Note that both Reflexion and LATS baselines demonstrate a significant performance gap when shifting from smaller to larger model, while the results for $\name$ is relatively consistent and maintains the performance edge throughout different sizes of models. This demonstrates that $\name$'s principled cross-task experience sharing mechanism also excels in tasks requiring complex reasoning.

\noindent\textbf{Conclusion} Our experiments across Alfworld, Webshop, and HotPotQA demonstrate that \name{} consistently outperforms state-of-the-art baselines in both task success rate and sample efficiency. Notably, \name{} achieves superior performance even with smaller models like Llama 3.1 8b, highlighting its efficiency and practicality for resource-constrained scenarios. These results validate the effectiveness of leveraging principled cross-task experiences sharing through our theoretically grounded selection strategy, confirming that \name{} enhances sequential reasoning capabilities across diverse tasks and model sizes.


\begin{figure}[t!]
    \centering
    \subfigure[Alfworld (Llama3.1 8b)]{
    \begin{minipage}{0.31\textwidth}
        \centering
        \includegraphics[width=\textwidth]{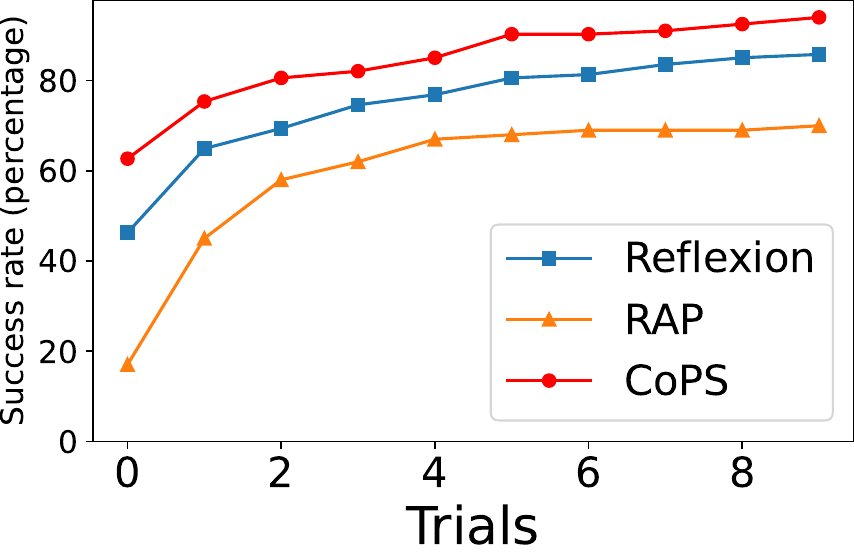} 
        \label{alf:8b}
    \end{minipage}
    }
    \subfigure[Webshop (Llama3.1 8b)]{
    \begin{minipage}{0.31\textwidth}
        \centering
        \includegraphics[width=\textwidth]{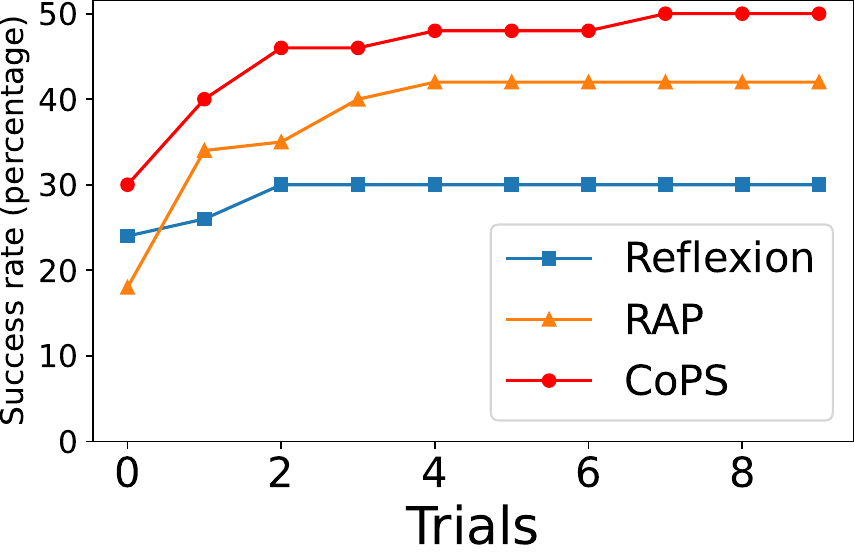}
        \label{web:8b}
    \end{minipage}
    }
    \subfigure[HotPotQA (Llama3.1 8b)]{
    \begin{minipage}{0.31\textwidth}
        \centering
        \includegraphics[width=\textwidth]{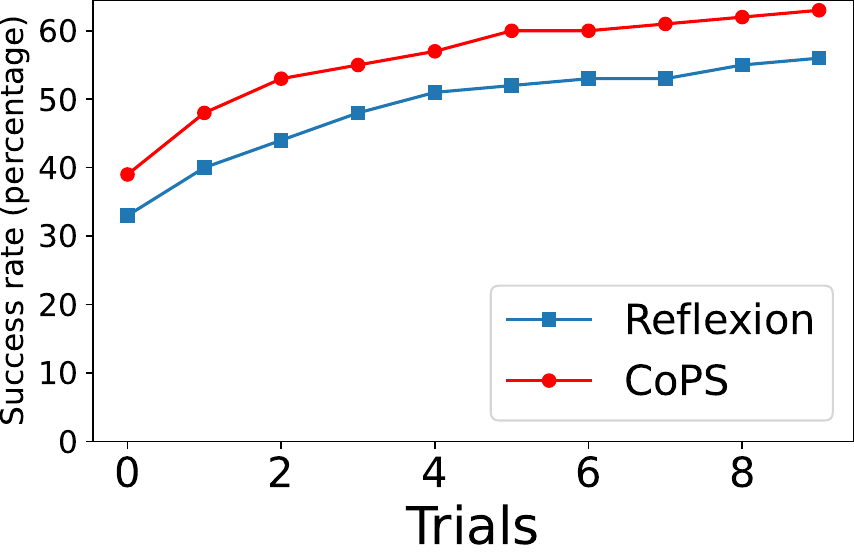} 
        \label{hot:8b}
    \end{minipage}
    }
    \subfigure[Alfworld (Llama3.1 70b)]{
    \begin{minipage}{0.31\textwidth}
        \centering
        \includegraphics[width=\textwidth]{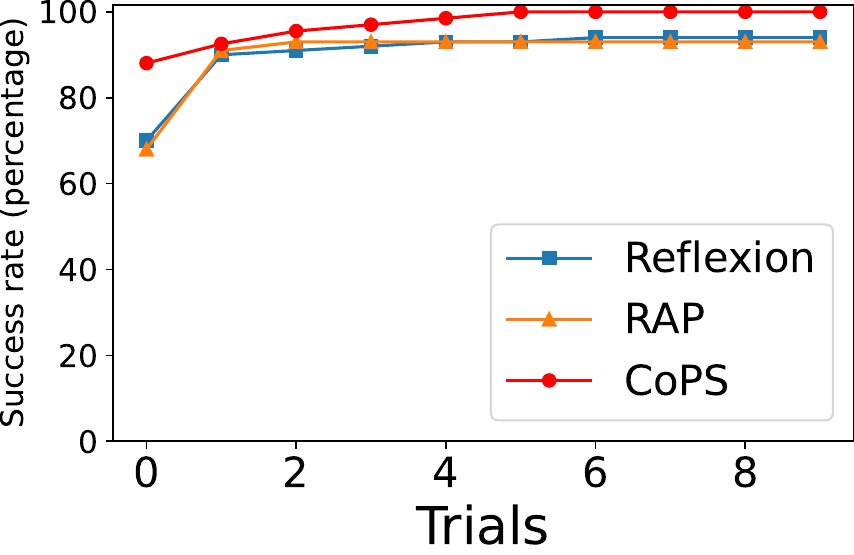} 
        \label{alf:70b}
    \end{minipage}
    }
    \subfigure[Webshop (Llama3.1 70b)]{
    \begin{minipage}{0.31\textwidth}
        \centering
        \includegraphics[width=\textwidth]{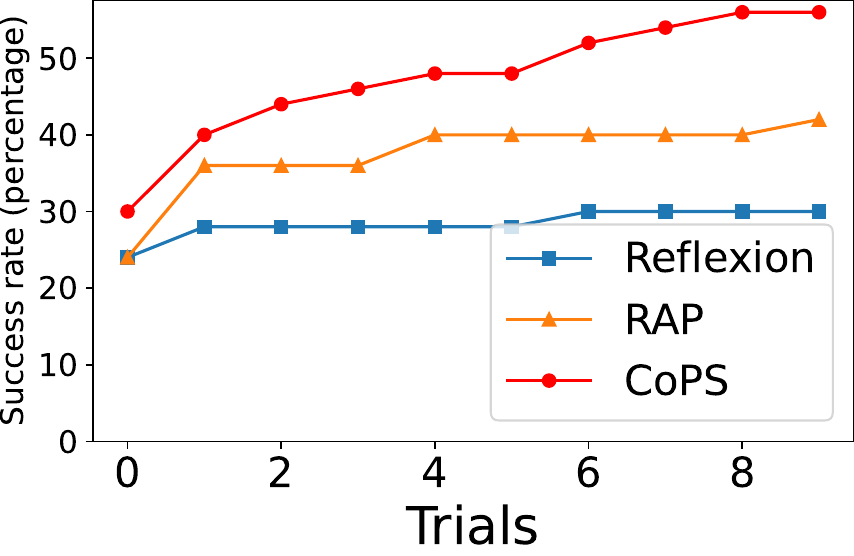} 
        \label{web:70b}
    \end{minipage}
    }
    \subfigure[HotPotQA (Llama3.1 70b)]{
    \begin{minipage}{0.31\textwidth}
        \centering
        \includegraphics[width=\textwidth]{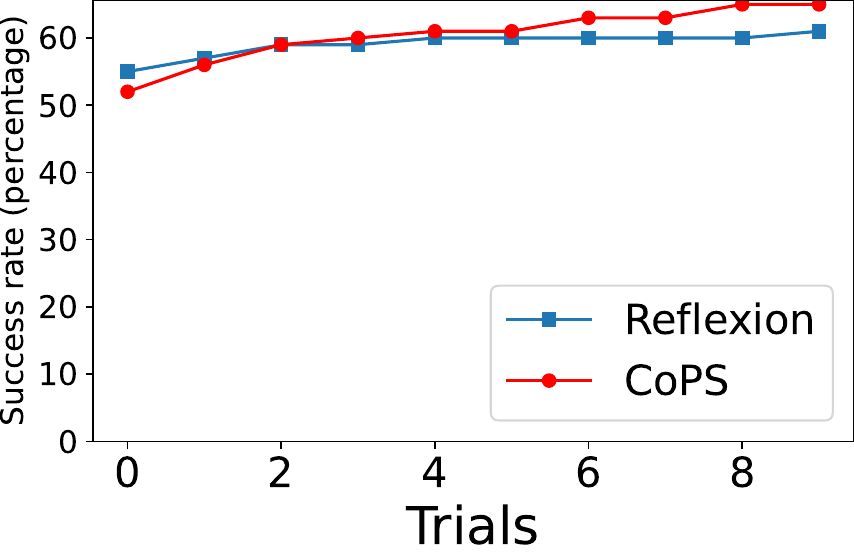} 
        \label{hot:70b}
    \end{minipage}
    }
    \vspace{0.5cm}
    \caption{Comparative evaluation of \name{}, Reflexion, RAP, and LATS across three benchmarks: Alfworld, Webshop, and HotPotQA. The figures illustrate the success rates for both the smaller Llama 3.1 8b and larger Llama 3.1 70b models, averaged over 10 trials.}\label{fig:allresult}
\end{figure}

\subsection{Ablation Study}\label{sec:ablation}

In this section, we analyze how two key hyperparameters affect the performance of \name{}: the scaling factor $c$ in Equation~(\ref{transform}) and the number of in-context experiences $k$ placed at the beginning of prompts. We conducted experiments on the Alfworld benchmark using both Llama 3.1 8b and Llama 3.1 70b models.

\noindent\textbf{Scaling factor} For the scaling factor $c$, we tested four settings: $c = 0$, $1$, $5$ and $10$, while keeping the number of in-context experiences fixed at $k = 5$ (see Figures~\ref{ablc:8b} and \ref{ablc:70b}). Our findings indicate that for smaller models like Llama 3.1 8b, a small but non-zero value of $c$ (e.g., $c = 1$) generally yields better performance (Figure~\ref{ablc:8b}). This suggests that moderate scaling effectively balances model adaptability and robustness on less capable models.

\noindent\textbf{Number of experiences} Regarding the number of in-context experiences $k$, we evaluated values ranging from $1$ to $10$, setting $c = 0$ (see Figures~\ref{ablk:8b} and \ref{ablk:70b}). We observed that performance improves as $k$ increases up to $k = 3$, after which it plateaus for both model sizes. This result indicates that while increasing the in-context experience size enhances performance to a point, adding more than three experiences may not offer substantial gains.

Our ablation study reveals that tuning key hyperparameters in \name{} is crucial for optimal performance. Specifically, for smaller models, a small but non-zero scaling factor \( c \) (e.g., \( c = 1 \)) effectively balances adaptability and robustness. Additionally, increasing the number of in-context experiences \( k \) enhances performance up to \( k = 3 \), beyond which additional experiences offer minimal gains. These insights provide practical guidance for hyperparameter selection, ensuring that \name{} can be efficiently deployed across various settings to maximize its sequential reasoning capabilities.

\begin{figure}[t!]
    \centering
    \subfigure[Llama3.1 8b]{
    \begin{minipage}{0.22\textwidth}
        \centering
        \includegraphics[width=\textwidth]{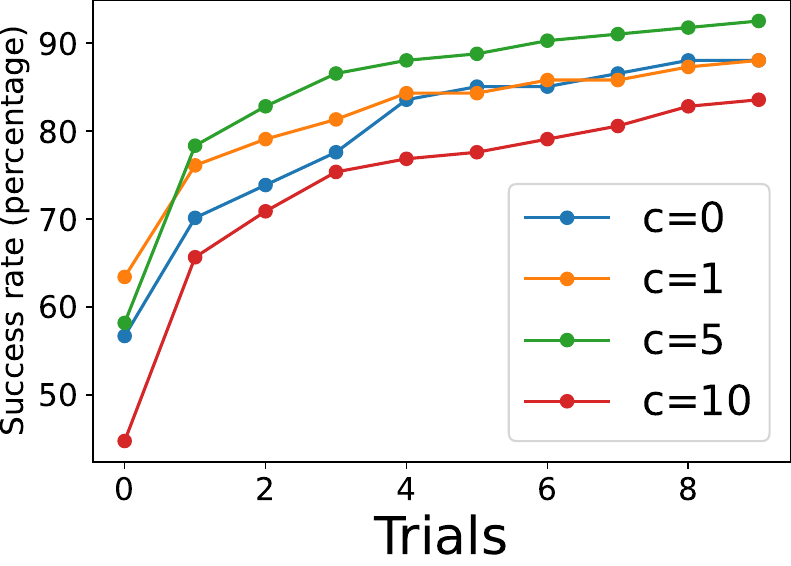} 
        \label{ablc:8b}
    \end{minipage}
    }
    \subfigure[Llama3.1 70b]{
    \begin{minipage}{0.22\textwidth}
        \centering
    \includegraphics[width=\textwidth]{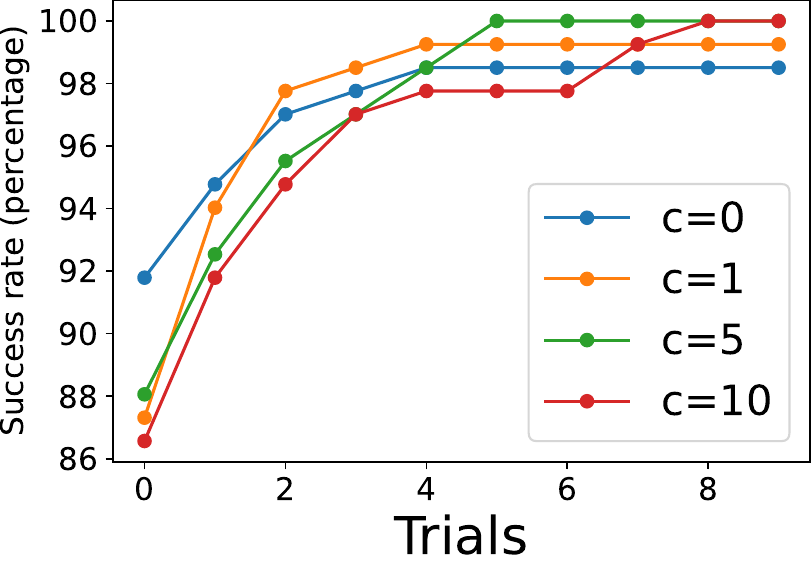}
    \label{ablc:70b}
    \end{minipage}
    }
    \subfigure[Llama3.1 8b]{
    \begin{minipage}{0.22\textwidth}
        \centering
        \includegraphics[width=\textwidth]{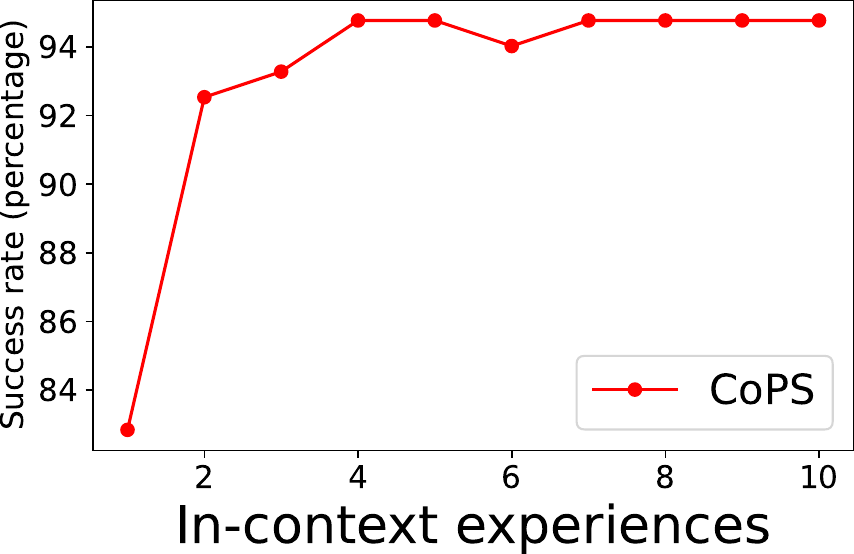} 
        \label{ablk:8b}
    \end{minipage}
    }
    \subfigure[Llama3.1 70b]{
    \begin{minipage}{0.22\textwidth}
        \centering
        \includegraphics[width=\textwidth]{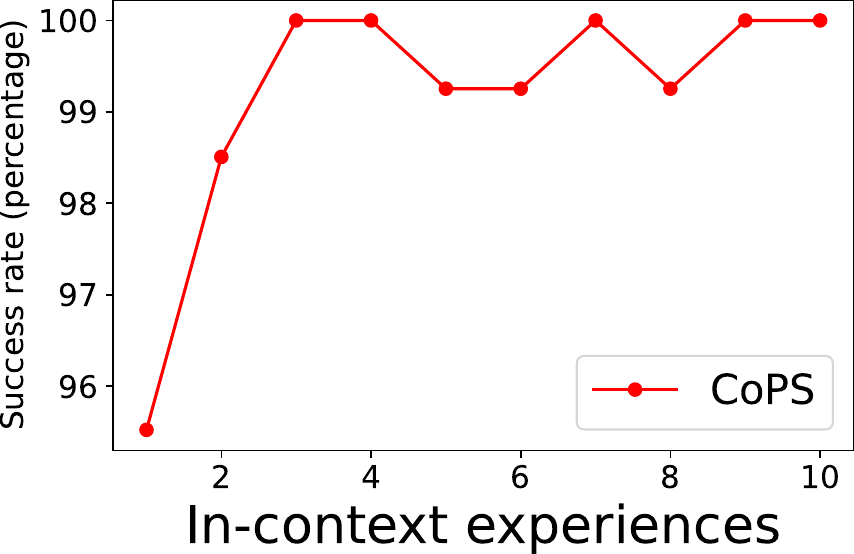} 
       \label{ablk:70b}
    \end{minipage}
    }
    \vspace{0.5cm}
    \caption{Performance impact of hyperparameters $c$ (scaling factor) and $k$ (number of in-context experiences) on the Alfworld benchmark for both Llama 3.1 8b and Llama 3.1 70b models.}\label{ablation:c}
\end{figure}
\section{Theoretical Framework of Experience-Assisted Agents}\label{sec:theory}

In this section, we develop the theoretical framework to demonstrate the effectiveness of \name. For simplicity, we analyze our algorithm in a bandit setting, where the maximum number of steps for each experience is $H = 1$. Slightly different from the formulation in Section \ref{sec:propmethod}, we define an experience as $\tau = s|a|r$, consisting of an initial state $s$, an action $a$, and its reward $r = r(s, a)$.

We introduce additional notations for clarity in our analysis. Let $\cT = \tau_1|\tau_2|\dots$ denote the experience collection. The length of $\cT$ is denoted by $|\cT|$, i.e., $\cT = (\tau_1,...,\tau_{|\cT|})$. We use $\cT_{t}$ to represent the first $t$ steps of the experience collection, i.e., $\cT_t = \tau_1|\dots|\tau_t$. For any experience collection $\cT$, we assume $|\cT| \leq T$. We define $\mathtt{T}$ as the space of all trajectories, and $\mathtt{T}_t$ as the space of trajectories of length $t$. We denote a general algorithm as $\alg(\cdot|\cdot, \cdot, \cdot): \mathtt{M} \times \mathtt{T} \times \mathtt{S} \rightarrow \Delta(\mathtt{A})$, which takes as input a task $\Mb \in \mathtt{M}$, an experience collection $\cT \in \mathtt{T}$, and a state $s \in \mathtt{S}$, and outputs a distribution over actions $a \in \mathtt{A}$. Note that some algorithms may not use the task $\Mb$ as input, in which case we write $\alg(\cdot|\cdot, \cdot)$. We denote $\PP^{\Mb, \alg}_t$ as the distribution over the first $t$ steps of an experience collection under task $\Mb$ and algorithm $\alg$. For an algorithm $\alg$ that takes $\Mb, \cT, s$ as input, we define its \emph{posterior average} as $\overline\alg(\cdot|\cT,s) = \EE_{\Mb\sim \PP^{\mathtt{M}}(\cdot|\cT' = \cT, s' = s)}[\alg(\cdot|\Mb, \cT', s')]$, which is the best Bayesian approximation of $\alg$ given the experience collection $\cT$ and current state $s$. 

\subsection{LLM Pretraining}

We begin by describing the pretraining process for the LLM. Let $\algtheta(\cdot|\cT, s): \mathtt{T} \times \mathtt{S} \rightarrow \Delta(\mathtt{A})$ represent an LLM agent that outputs a distribution over $\mathtt{A}$, where $\hat\btheta\in \bTheta$ is the parameter of the LLM, $\bTheta$ denotes the whole parameter space. We assume that there exists a pretraining dataset $\cD_{\text{pre}} = \{\cT^1, \dots, \cT^{\npre}\}$, with $|\cT^i| = T-1$. Following the pretraining setup in \citet{lin2023transformers}, we assume two algorithms: a \emph{context algorithm}, $\algc(\cdot|\cdot, \cdot): \mathtt{T} \times \mathtt{S} \rightarrow \Delta(\mathtt{A})$, and an \emph{expert algorithm}, $\alge(\cdot|\cdot, \cdot, \cdot): \mathtt{M} \times \mathtt{T} \times \mathtt{S} \rightarrow \Delta(\mathtt{A})$. In general, the context algorithm provides a ``natural" action based on the experience collection and current state, while the expert algorithm provides a more informed action, given the task information, experience collection, and current state. Since the expert algorithm has access to task information $\Mb$, it typically produces better actions than the context algorithm.

We now describe the pretraining process. To generate an experience collection $\cT = \tau_1|\dots|\tau_{T-1} \in \cD_{\text{pre}}$, we first sample a task $\Mb \sim \PP^{\mathtt{M}}$. For each experience $\tau_i$, the state is sampled from the initial state distribution $s_i \sim \PP^{\Mb}_0$, the action is sampled using the context algorithm $a_i \sim \algc(\cdot|\cT_{i-1}, s_i)$, and the reward is given by $r_i = r(s_i, a_i)$. After generating the experience collection, we collect expert feedback $\bar a_1, \dots, \bar a_{T-1}$ for each step of $\cT$, using the expert algorithm, where $\bar a_i \sim \alge(\cdot|\Mb, \cT_{i-1}, s_i)$. Repeating this process $\npre$ times produces the trajectories $\cT^i$ and expert actions $\bar a_1^i, \dots, \bar a_{T-1}^i$ for $i \in [\npre]$. Finally, we pretrain the LLM $\algtheta$ by solving the following maximum likelihood estimation problem: \begin{align} \hat\btheta \leftarrow \argmax_{\btheta\in\bTheta} \sum_{i=1}^{\npre} \sum_{t=1}^{T} \log \alg_{\btheta}(\bar a_t^i|\cT^i_{t-1
}, s_t^i).\notag \end{align}

For the remainder of this paper, we use $\algtheta$ to represent our LLM. Below, we present several standard assumptions for analyzing $\algtheta$. 

\begin{definition}[\citealt{lin2023transformers}]\label{def:1} Let $\bTheta$ be the set of parameters of the LLM, $\alg_{\btheta}$. We call $\bTheta_0 \subseteq \bTheta$ a $\rho$-cover of $\bTheta$ with respect to $\alg_{\btheta}$ if, for any $\btheta \in \bTheta$, there exists $\btheta_0 \in \bTheta_0$ such that \begin{align} 
\forall s \in \mathtt{S}, t \in [T], \cT \in \mathtt{T}_{t-1}, \|\log \alg_{\btheta}(\cdot|\cT, s) - \log \alg_{\btheta_0}(\cdot|\cT, s)\|_\infty \leq \rho.\notag 
\end{align} 
We denote $\cN(\rho) = |\bTheta_0|$ as the $\rho$-covering number of $\alg_{\btheta}$. 
\end{definition}
Next assumption assumes that, the best approximation between the trained LLM and the posterior average of the expert algorithm, $\overline{\alge}$, can be bounded by some constant. 
\begin{assumption}[\citealt{lin2023transformers}]\label{ass:1} There exists $\btheta^* \in \bTheta$ and a \emph{model capacity error} $\ereal > 0$ such that \begin{equation} 
\forall t \in [T], \log \EE_{\Mb \sim \PP^{\mathtt{M}}, s \sim \PP_0^{\Mb}, \cT \sim \PP_{t-1}^{\Mb, \algc}, \bar a \sim \overline{\alge}(\cdot|\cT, s)}\bigg[\frac{\overline{\alge}(\bar a|\cT, s)}{\alg_{\btheta^*}(\bar a|\cT, s)}\bigg] \leq \ereal.\notag \end{equation} 
\end{assumption}

Finally, we make assumptions for the decoder $\decoder$ introduced in Algorithm \ref{alg:proposed_method}. We assume access to a class of decoders $\decoder_t: \mathtt{S}\rightarrow \Delta(\mathtt{T}_{t})$ that maps the state $s$ to a distribution over the space of $t$ number of experiences, capable of estimating the distribution $\PP^{\Mb, \algc}_t(\cT)$, which represents the task-dependent experience distribution offered by LLM.

\begin{assumption}\label{ass:decoder} 
For the decoder $\decoder_t: \mathtt{S} \rightarrow \Delta(\mathtt{T}_t)$, there exists a \emph{decoder coefficient} $C_{\decoder} > 1$ such that for any $t \in [T], \cT \in \mathtt{T}_{t-1}, \Mb \in \mathtt{M}$ and $s \sim \PP_0^{\Mb}$, we have
\begin{equation} 
\frac{1}{C_{\decoder}^2} \leq \frac{\decoder_{t-1}(\cT|s)}{\PP^{\Mb, \algc}_{t-1}(\cT)} \leq C_{\decoder}^2. \notag
\end{equation} 
\end{assumption}

\subsection{Algorithms for the Offline Setting}

We consider the same offline setting as in Section \ref{sec:propmethod}. Suppose we have an offline dataset $\mathcal{D}$, and the agent is given an initial state $s$. We formalize the experience selection problem as a distribution selection problem, where the agent has access to a candidate set of distributions, denoted by $\mathcal{P} = \{\mathbb{P}^{1}(\cdot|\cdot, \cdot), \dots, \mathbb{P}^{|\mathcal{P}|}(\cdot|\cdot, \cdot)\} \subseteq 2^{\mathtt{T}_{T-1} \times \mathtt{S} \rightarrow \Delta(\mathtt{T}_{T-1})}$. Each element in this set represents a mapping from the dataset $\mathcal{D}$ and the current state $s$ to a distribution over trajectories $\mathcal{T}$ of length $T-1$. In general, each $\mathbb{P}^i$ can be interpreted as the distribution over all possible combinations of $T-1$ experiences from the dataset $\mathcal{D}$. The agent’s goal is to select a distribution $\ps$ from $\mathcal{P}$ that minimizes the suboptimality gap, which quantifies the performance difference between the best possible strategy and the strategy selected by the agent, as measured by the expert algorithm:
\begin{align}
    \text{SubOpt}(\ps):=\EE_{\Mb \sim \PP^{\mathtt{M}}, s \sim \PP_0^{\Mb}}\bigg[\max_{\phat \in \cP}\EE_{ \cT \sim \phat, a\sim \overline{\alge}(\cdot|\cT, s)}r(s,a) - \EE_{ \cT \sim \ps, a\sim \overline{\alge}(\cdot|\cT, s)}r(s,a)\bigg].\label{def:offlinegoal}
\end{align}


We propose \texttt{\offlinename} in Algorithm \ref{alg:offlinegen}, which is an experience collection-based version of \texttt{\name}. The core idea of \texttt{\offlinename} mirrors that of \texttt{\name}: the agent seeks to find experience collection that maximize the reward while minimizing the distributional shift from the experience collection of the current task, denoted by LLM. Given the test state $s$, \texttt{\offlinename} first runs the decoder to obtain a distribution $\decoder_{T-1}(\cdot|s)$, which approximates $\mathbb{P}^{\mathcal{M}, \algc}_{t-1}$. Then, \texttt{\offlinename} applies the \emph{pessimism principle}, as in \eqref{help:123}. The selected distribution $\mathbb{P}^* \in \mathcal{P}$ aims to identify a distribution that produces an experience collection which maximizes the reward given the actions provided by the LLM, while staying close to the decoded distribution $\decoder_{T-1}(\cdot|s)$. To measure the distributional distance, we employ the $\chi^2$-distance. Similar to the hyperparameter $c$ in \texttt{\name}, \texttt{\offlinename} introduces a hyperparameter $\epsilon_{\text{pre}}$ to balance the trade-off between maximizing reward and satisfying the regularity condition imposed by $\decoder_{T-1}(\cdot|s)$.

\begin{algorithm}[t!]
\caption{$\offlinename$ }\label{alg:offlinegen}
\begin{algorithmic}[1]
\REQUIRE LLM $\algtheta(\cdot|\cdot, \cdot)$, candidate experience collection distribution $\cP$, pretraining error parameter $\epre$, task decoder $\decoder$, offline dataset $\cD$.
\STATE Receive test state $s$, decode the distribution $\decoder_{T-1}(\cdot|s)$. 
\STATE Select $\ps$ from $\cP$ that maximizes the following:
\begin{align}
&\ps = \argmax_{\phat \in \cP}\EE_{\substack{\cT \sim \phat(\cdot|\cD,s),\\ a\sim \algtheta(\cdot|\cT, s)} }r(s,a) - \epre\sqrt{1+\chi^2(\phat(\cdot|\cD, s), \decoder_{T-1}(\cdot|s))}.\label{def:pess}
    \end{align}
\STATE Generate $\cT^s\sim \ps$ and obtain  $a\sim \algtheta(\cdot|\cT^s, s)$.
\end{algorithmic}
\end{algorithm}

We have the following theorem to characterize the performance of $\offlinename$. 
\begin{theorem}\label{thm:1}
By setting 
\begin{align}
    \epre = C_{\decoder}T\cdot\sqrt{5\cdot T\log(\cN(1/(\npre T)^2)T)\cdot \npre^{-1} + T\ereal},\notag
\end{align}
and denote $\PP^{*,s} = \argmax_{\phat\in\cP}\EE_{ \cT \sim \phat(\cdot|\cD, s), a\sim \overline{\alge}(\cdot|\cT, s)}r(s,a)$, we have the following bound with probability at least $1-2/T$:
    \begin{align}
        \text{SubOpt}(\ps)
        \leq 2C_{\decoder}\epre \EE_{\Mb \sim \PP^{\mathtt{M}}, s \sim \PP_0^{\Mb}}\sqrt{1+\chi^2( \PP^{*,s}(\cdot|\cD, s), \PP^{\Mb, \algc}_{T-1}(\cdot))}.\notag
    \end{align}
\end{theorem}
\begin{proof}
    See Appendix \ref{proof:thm:1}. 
\end{proof}

Theorem \ref{thm:1} provides several insights into why \texttt{\name} achieves superior performance and how experience selection should be tailored under different circumstances:

\begin{itemize}[leftmargin=*]
    \item The final suboptimality gap of the selected distribution $\mathbb{P}^{*,s}$ depends on the decoder coefficient $C_{\decoder}$ and the pretraining error parameter $\epsilon_{\text{pre}}$. This implies that for a more powerful LLM, the selected experience distribution $\mathbb{P}^{*,s}$ will be closer to the optimal distribution. Meanwhile, the dependence of $\PP^{\Mb, \algc}_{T-1}$ suggests that the task-dependent experience collection distribution offered by LLM serves as a strong regularizer to select the optimal retrieval strategy.
    \item The optimal choice of the pretraining error parameter $\epsilon_{\text{pre}}$ is influenced by the decoder coefficient $C_{\decoder}$, the number of pretraining trajectories in the pretraining set $\npre$, and the model capacity error $\epsilon_{\text{real}}$. In general, for a more powerful LLM, where $\npre$ is large and $\epsilon_{\text{real}}$ is small, our theorem suggests that the agent should focus more on aligning the selected experience collection distribution $\mathbb{P}^{*,s}$ with the decoder distribution $\decoder$. This aligns with our observations in Section \ref{sec:ablation}, where smaller models, such as LLaMA 3.1 8b, are more sensitive to the choice of the hyperparameter $c$.
\end{itemize}


\subsection{Algorithms for the Online Setting}\label{sec:online}
\newcommand{\algonline}{{\text{alg}}}
\newcommand{\algts}{{\text{alg}^{\text{TS}}}}
\newcommand{\ptstar}{{\PP^{*,t}}}
\newcommand{\pt}{{\hat\PP^t}}

We also consider an analysis for a variant of $\offlinename$ to the online setting. Here, let $\cP = \{\PP^{1}(\cdot|\cdot, \cdot),\dots, \PP^{|\cP|}(\cdot|\cdot, \cdot)\}\subseteq 2^{\mathtt{T}_{t-1}\times \mathtt{S}\rightarrow \Delta(\mathtt{T}_{t-1})}$ which includes mappings that map an experience collection $\cT_{t-1}$ and a test state $s$ to a distribution over $\mathtt{T}_{t-1}$.  Each $\PP^i$ can be thought as a strategy to pick the experience collection that depends on the past observations. At step $t$, we have history $\cH_{t-1} = \{s_1, a_1, r_1,\dots, s_{t-1}, a_{t-1}, r_{t-1}\}$. Then the agent receives $s_t \sim \PP_0^{\Mb_t}$, where $\Mb_t\sim \PP^{\mathtt{M}}$. Then the agent selects $\PP_t$ by some algorithm and samples $\cT_{t-1}\sim \PP_t(\cdot|\cH_{t-1}, s_t)$. Then the agent takes the action $a_t\sim \algtheta(\cdot|\cT_{t-1}, s_t)$. Her goal is to minimize the following regret:  
\begin{align}
    \text{Regret}_T:=\sum_{t=1}^T\EE_{\Mb_t \sim \PP^{\mathtt{M}}, s_t \sim \PP_0^{\Mb_t}}\bigg[\max_{\PP^i \in \cP}\EE_{ \substack{\cT_{t-1} \sim \PP^i(\cdot|\cH_{t-1}),\\ \bar a\sim \overline{\alge}(\cdot|\cT_{t-1}, s_t)}}r(s_t,\bar a) - \EE_{ \substack{\cT_{t-1} \sim \PP_t(\cdot|\cH_{t-1}), \\a_t\sim \overline{\alge}(\cdot|\cT_{t-1}, s_t)}}r(s_t,a_t)\bigg].\label{def:onlinegoal}
\end{align}

We propose the algorithm $\onlinename$ in Algorithm \ref{alg:onlinegen}. Similar to $\offlinename$, $\onlinename$ adapts an decoder that takes the current state as its input and outputs a distribution of the experience collection $\cT$, which aims to estimate the LLM output distribution $\PP_{t-1}^{\Mb_t, \algc}$. Unlike $\offlinename$, the optimization goal of $\onlinename$ in \eqref{def:optim} is similar to the \emph{optimistic principle} that originates from the online decision-making problems \citep{abbasi2011improved}, which aims to maximize both the reward and the distribution distance between the decoder distribution $\decoder_{t-1}$ and the selected one $\pt$. Meanwhile, note that the selected experience collection distribution only depends on the past history $\cH_{t-1}$, which is small in the early stage of the online decision-making process. We have the following theorem to demonstrate the theoretical guarantee of $\onlinename$.

\begin{algorithm}
\caption{$\onlinename$}\label{alg:onlinegen}
\begin{algorithmic}[1]
\REQUIRE LLM $\algtheta(\cdot|\cdot, \cdot)$, candidate experience collection distribution $\cP$, pretraining error parameter $\epre$, task decoder $\decoder$.
\STATE Let $\cH_0 = \emptyset$.
\FOR{$t = 1,\dots, T$}
\STATE Generate $\Mb_t \sim \PP^{\mathtt{M}}$, receive $s_t \sim \PP_0^{\Mb_t}$, decode $\decoder_{t-1}(\cdot|s_t)$
\STATE Select $\pt$ from $\cP$ that maximizes the following:
\begin{small}
    \begin{equation}
\pt = \argmax_{\phat \in \cP}\EE_{ \substack{\cT \sim \phat(\cdot|\cH_{t-1}, s_t), \\a\sim \algtheta(\cdot|\cT, s_t)}}r(s_t,a) + \epre\sqrt{1+\chi^2(\phat(\cdot|\cH_{t-1}, s_t), \decoder_{t-1}(\cdot|s_t))}.\label{def:optim}
    \end{equation}
\end{small}
\STATE Generate $\cT\sim \pt(\cdot|\cH_{t-1}, s_t)$ and obtain $a_t\sim \algtheta(\cdot|\cT, s_t)$ and $r_t = r(s_t, a_t)$, set $\cH_t = \cH_{t-1}\cup (s_t, a_t, r_t)$.
\ENDFOR 
\end{algorithmic}
\end{algorithm}

\begin{theorem}\label{thm:2}
By setting 
\begin{align}
    \epre = C_{\decoder}\cdot T^2\cdot\sqrt{5\cdot \frac{T\log(\cN(1/(\npre T)^2)T^2)}{\npre} + T\ereal},\notag
\end{align}
and denote $\PP^{*,t} = \argmax_{\hat \PP \in \cP}\EE_{ \substack{\cT_{t-1} \sim \hat \PP(\cdot|\cH_{t-1}, s_t),\\ \bar a\sim \overline{\alge}(\cdot|\cT_{t-1}, s_t)}}r(s_t,\bar a)$, we have the following bound holds with probability at least $1-2/T$:  
    \begin{align}
        \text{Regret}_T
        \leq 2C_{\decoder}\epre\sum_{t=1}^T\sqrt{1+\chi^2(\PP^{*,t}(\cdot|\cH_{t-1}, s_t), \PP^{\Mb_{t}, \algc}_{t-1}(\cdot))}.\notag
    \end{align}
\end{theorem}
\begin{proof}
    See Appendix \ref{thm:75}. 
\end{proof}
Similar to Theorem \ref{thm:1} for the offline setting, Theorem \ref{thm:2} also shares the following insights. 
\begin{itemize}[leftmargin = *]
    \item The regret is controlled by the difference between the best experience collection generated distribution $\ptstar$ and the experience collection distribution induced by the contextual algorithm at $t$-th step. Therefore, the best strategy overall is to select trajectories from the history $\cH_{t-1}$ that can approximates the current task well to avoid the distribution shift. 
    \item With a more powerful LLM, the $\epre$ will be smaller, which means the selected experience collection can approximate the best selection better. 
\end{itemize}

\section{Conclusion, Limitation and Future Work}

In this paper, we introduced \name{} (\textbf{\underline{C}}r\textbf{\underline{o}}ss-Task Ex\textbf{\underline{p}}erience \textbf{\underline{S}}haring), a theoretically
grounded algorithm that empowers agent systems with cross-task experiences sharing. Using a pessimism-based strategy to select relevant experiences, \name{} maximizes utility while minimizing the risks of distribution shifts. Our experiments on benchmarks like Alfworld, Webshop, and HotPotQA demonstrate that \name{} outperforms state-of-the-art methods in both success rates and sample efficiency. Theoretically, we show that our algorithm's performance depends on the LLM’s pre-trained quality and the matching between the cross-task experience distribution decided by the trials selected by the agent, and a task-dependent experience distribution denoted by the LLM, providing insights for improving experience retrieval methods.

While \name{} shows clear improvements over existing methods, it has several limitations. Its effectiveness heavily depends on the quality and diversity of the experiences in the memory bank, meaning that outdated or poorly aligned experiences can reduce its performance. Additionally, \name{} is sensitive to hyperparameters like the scaling factor and the number of in-context experiences, which may require time-consuming tuning that doesn’t always generalize well across different tasks or models. Finally, the theoretical guarantees we provide also rely on assumptions about the accuracy of the decoder and specific pre-training properties of the LLM, which may not always hold in real-world scenarios. 

Looking ahead, several research directions could further improve \name{}.
These include developing adaptive hyperparameter tuning methods, exploring dynamic memory management to keep experiences relevant, and adding ways to assess experience quality. Additionally, adapting \name{} for multi-agent systems and integrating it with reinforcement learning could enable collaborative learning and continuous performance improvement. Pursuing these areas will help LLM agents handle even more complex sequential reasoning tasks.


\appendix

\section{Additional Details in Section \ref{sec:theory}}
\subsection{Proof of Theorem \ref{thm:1}}\label{proof:thm:1}
We prove Theorem \ref{thm:1} here. First, we need the following lemmas. 
\begin{lemma}[Lemma 20, \citealt{lin2023transformers}]\label{lemma:linpretrain}
With probability at least $1-\delta$, we have
\begin{align}
    &\EE_{\Mb \sim \PP^{\mathtt{M}}, s \sim \PP_0^{\Mb}, \cT \sim \PP^{\Mb, \algc}_{T-1}}\bigg[\sum_{t=1}^T\disth^2(\overline{\alge}(\cdot|\cT_{t-1}, s), \algtheta(\cdot|\cT_{t-1}, s))\bigg] \notag \\
    & \leq 5\cdot \frac{T\log(\cN(1/(\npre T)^2)T/\delta)}{\npre} + T\ereal, \notag
\end{align}
    where the covering number $\cN$ is defined in Definition \ref{def:1}, $\ereal$ is defined in Assumption \ref{ass:1}. 
\end{lemma}

Next lemma is used to provide a per-state guarantee for the generalization error. 
\begin{lemma}\label{lemma:pretrain}
Let event $\event$ be defined as
\begin{align}
    &\EE_{\cT \sim \PP^{\Mb, \algc}_{T-1}}\bigg[\sum_{t=1}^T\disth^2(\overline{\alge}(\cdot|\cT_{t-1}, s), \algtheta(\cdot|\cT_{t-1}, s))\bigg]  \leq m_c\bigg[c\cdot \frac{T\log(\delta^{-1}\cN(1/(\npre T)^2)T)}{\npre} + T\ereal\bigg], \notag
\end{align}
    where $\ereal$ is defined in Assumption \ref{ass:1}. Then we have $\PP(\cE)\geq 1-1/m_c - \delta$. 
\end{lemma}
\begin{proof}
    By Markov inequality, we have that with probability at most $1/m_c$, 
    \begin{align}
        &\EE_{\cT \sim \PP^{\Mb, \algc}_{T-1}}\bigg[\sum_{t=1}^T\disth^2(\overline{\alge}(\cdot|\cT_{t-1}, s), \algtheta(\cdot|\cT_{t-1}, s))\bigg]\notag \\
        & \geq m_c\cdot\EE_{\Mb \sim \PP^{\mathtt{M}}, s \sim \PP_0^{\Mb}, \cT \sim \PP^{\Mb, \algc}_{T-1}}\bigg[\sum_{t=1}^T\disth^2(\overline{\alge}(\cdot|\cT_{t-1}, s), \algtheta(\cdot|\cT_{t-1}, s))\bigg]. \notag
    \end{align}
    Meanwhile, by Lemma \ref{lemma:linpretrain}, we know that with probability at most $\delta$, we have
    \begin{align}
        &\EE_{\Mb \sim \PP^{\mathtt{M}}, s \sim \PP_0^{\Mb}, \cT \sim \PP^{\Mb, \algc}_{T-1}}\bigg[\sum_{t=1}^T\disth^2(\overline{\alge}(\cdot|\cT_{t-1}, s), \algtheta(\cdot|\cT_{t-1}, s))\bigg] \notag \\
        &\geq c\cdot \frac{T\log(\delta^{-1}\cdot\cN(1/(\npre T)^2)T)}{\npre} + T\ereal.\notag
    \end{align}
    Therefore, by the union bound, we have $\PP(\event)\geq 1-\delta - 1/m_c$.
\end{proof}

Now we begin to prove Theorem \ref{thm:1}. 
\begin{proof}
We follow the proof steps in \citet{lin2023transformers}. We suppose that the event $\event$ denoted in Lemma \ref{lemma:pretrain} holds. We first bound the difference of reward by the difference between their distribution distance. Let $\phat$ be an arbitrary distribution over $\cT$. Then we have 
    \begin{align}
        &\EE_{ \cT \sim \phat(\cdot)}|\EE_{a\sim \overline{\alge}(\cdot|\cT, s)} r(s,a) - \EE_{a\sim \algtheta(\cdot|\cT, s)} r(s,a)| \notag \\
        &\leq \EE_{ \cT \sim \phat(\cdot)}\disttv (\overline{\alge}(\cdot|\cT, s), \algtheta(\cdot|\cT, s))\notag \\
        & \leq \EE_{ \cT \sim \phat(\cdot)}\disth(\overline{\alge}(\cdot|\cT, s), \algtheta(\cdot|\cT, s)),\label{eq:110}
    \end{align}
    the first inequality holds due to the fact $|r| \leq 1$ and the property of TV distance, the second one holds since $\disttv \leq \disth$. Starting from \eqref{eq:110} we have     
\begin{align}
        &\EE_{ \cT \sim \phat}\disth(\overline{\alge}(\cdot|\cT, s), \algtheta(\cdot|\cT, s)) \notag \\
        &= \EE_{ \cT \sim \PP^{\Mb, \algc}_{T-1}}\disth(\overline{\alge}(\cdot|\cT, s), \algtheta(\cdot|\cT, s))\cdot \frac{\phat(\cT)}{\PP^{\Mb, \algc}_{T-1}(\cT)}\notag \\
        & \leq \sqrt{\underbrace{\EE_{ \cT \sim \PP^{\Mb, \algc}_{T-1}}\disth^2(\overline{\alge}(\cdot|\cT, s), \algtheta(\cdot|\cT, s))}_{I_1}} \cdot \sqrt{\underbrace{\EE_{ \cT \sim \PP^{\Mb, \algc}_{T-1}}\bigg(\frac{\phat(\cT)}{\PP^{\Mb, \algc}_{T-1}(\cT)}\bigg)^2}_{I_2}},\label{eq:111}
    \end{align}
    where the first inequality holds due to Cauchy-Schwarz inequality. For $I_1$, we use Lemma \ref{lemma:linpretrain}. Notice that the length of $|\cT| = T-1$ and the definition of $\epre$, we have
    \begin{align}
        I_1 \leq (\epre/C_{\decoder})^2.\label{eq:112}
    \end{align}
    For $I_2$, by the definition of $\chi^2$ distance, we have
    \begin{align}
        I_2 &= \EE_{ \cT \sim \phat}\frac{\phat(\cT)}{\PP^{\Mb, \algc}_{T-1}(\cT)} \notag \\
        &=\EE_{ \cT \sim \phat}\frac{\phat(\cT)}{\decoder_{T-1}(\cT|s)} \cdot \frac{\decoder_{T-1}(\cT|s)}{\PP^{\Mb, \algc}_{T-1}(\cT)}\notag \\ & \leq C_{\decoder}^2[1+\chi^2(\phat(\cdot), \decoder_{T-1}(\cdot|s))].\label{eq:113}
    \end{align}
    where the inequality holds due to Assumption \ref{ass:decoder}. 
    Substituting \eqref{eq:112} and \eqref{eq:113} into \eqref{eq:111}, and substituting \eqref{eq:111} into \eqref{eq:110}, we have
    \begin{align}
        &|\EE_{ \cT \sim \phat, a\sim \overline{\alge}(\cdot|\cT, s)}r(s,a) - \EE_{ \cT \sim \phat, a\sim \algtheta(\cdot|\cT, s)} r(s,a)|\notag \\
        & \leq \epre\sqrt{1+\chi^2(\phat(\cdot), \decoder_{T-1}(\cdot|s))}, \label{help:1}
    \end{align}
    holds for any $\phat \in \cP$. Finally, we have 
    \begin{align}
        &\EE_{ \cT^s \sim \ps(\cdot|\cD,s), a\sim \overline{\alge}(\cdot|\cT^s, s)}r(s,a) \notag \\
        &\geq \EE_{ \cT^s \sim \ps(\cdot|\cD,s), a\sim \algtheta(\cdot|\cT^s, s)}r(s,a) - \epre\sqrt{1+\chi^2(\ps(\cdot|\cD,s), \decoder_{T-1}(\cT|s))} \notag \\
        &\geq \EE_{ \cT^s \sim \PP^{*,s}(\cdot|\cD,s), a\sim \algtheta(\cdot|\cT^s, s)}r(s,a) - \epre\sqrt{1+\chi^2( \PP^{*,s}(\cdot|\cD,s), \decoder_{T-1}(\cT|s))},\notag\\
        & \geq \EE_{ \cT^s \sim \PP^{*,s}(\cdot|\cD,s), a\sim \overline{\alge}(\cdot|\cT^s, s)}r(s,a) - 2\epre\sqrt{1+\chi^2( \PP^{*,s}(\cdot|\cD,s), \decoder_{T-1}(\cT|s))},\notag\\
        & \geq \EE_{ \cT^s \sim \PP^{*,s}(\cdot|\cD,s), a\sim \overline{\alge}(\cdot|\cT^s, s)}r(s,a) - 2\epre C_{\decoder}\sqrt{1+\chi^2( \PP^{*,s}(\cdot|\cD,s), \PP^{\Mb, \algc}_{T-1}(\cdot))},\notag
    \end{align}
where the first inequality holds due to \eqref{help:1}, the second one holds due to the selection rule of $\ps$, the third one holds due to \eqref{help:1} and the last one holds due to Assumption \ref{ass:decoder}. This concludes our proof.

\end{proof}



\subsection{Proof of Theorem \ref{thm:2}} \label{thm:75}
\begin{proof}
Suppose we are at step $t$ and we condition on all past history $\cH_{t-1} = (s_1, a_1, r_1,\dots, s_{t-1}, a_{t-1}, r_{t-1})$. Let $\Mb_t$ be the task at $t$ step and $s_t$ be the state observed. Then with probability at least $1-1/m_c - \delta$, the following event $\cE_t$ holds:
\begin{align}
    &\EE_{\cT \sim \PP^{\Mb_t, \algc}_{t-1}}\bigg[\disth^2(\overline{\alge}(\cdot|\cT_{t-1}, s_t), \algtheta(\cdot|\cT_{t-1}, s_t))\bigg] \notag \\
    & \leq m_c\bigg[c\cdot \frac{T\log(\delta^{-1}\cN(1/(\npre T)^2)T^2)}{\npre} + T\ereal\bigg], \notag
\end{align}

Now following \eqref{eq:111} in the proof of Theorem \ref{thm:1}, we still have 
\begin{align}
    &\EE_{ \cT \sim \phat}|\EE_{a\sim \overline{\alge}(\cdot|\cT, s_t)} r(s,a) - \EE_{a\sim \algtheta(\cdot|\cT, s_t)} r(s,a)| \notag \\
    &\leq \sqrt{\underbrace{\EE_{ \cT \sim \PP^{\Mb_t, \algc}_{t-1}}\disth^2(\overline{\alge}(\cdot|\cT, s_t), \algtheta(\cdot|\cT, s_t))}_{I_1}} \cdot \sqrt{\underbrace{\EE_{ \cT \sim \PP^{\Mb_t, \algc}_{t-1}}\bigg(\frac{\phat(\cT)}{\PP^{\Mb_t, \algc}_{t-1}(\cT)}\bigg)^2}_{I_2}}
\end{align}
Then following Lemma \ref{lemma:pretrain}, under event $\cE_t$, we have
\begin{align}
I_1 \leq (\epre/C_{\decoder})^2,\  \epre/C_{\decoder} = T^2\cdot\sqrt{c\cdot \frac{T\log(\cN(1/(\npre T)^2)T^2)}{\npre} + T\ereal}.\notag
\end{align}
For $I_2$, similar to \eqref{eq:113}, we have
\begin{align}
    I_2 \leq C_{\decoder}^2[1+\chi^2(\phat(\cdot), \decoder_{t-1}(\cdot|s_t))].\notag
\end{align}
Therefore, we have for any $\phat$, 
\begin{align}
        &|\EE_{ \cT \sim \phat(\cdot|\cH_{t-1}, s_t), a\sim \overline{\alge}(\cdot|\cT, s_t)}r(s_t,a) - \EE_{ \cT \sim \phat(\cdot|\cH_{t-1}, s_t), a\sim \algtheta(\cdot|\cT, s_t)} r(s_t,a)|\notag \\
        &\leq \epre\sqrt{1+\chi^2(\phat(\cdot|\cH_{t-1}, s_t), \decoder_{t-1}(\cdot|s_t))}.\label{online:1}
    \end{align}
    Taking union bound and let $m_c = T^2, \delta = 1/T^2$, then we get $\cE_1,...,\cE_T$ hold with probability at least $1-2/T$.  Next we bound the suboptimal gap at $t$ step as follows:

\begin{align}
    &\EE_{ \cT^{t-1} \sim \pt(\cdot|\cH_{t-1}, s_t), a\sim \overline{\alge}(\cdot|\cT^{t-1}, s_t)}r(s_t,a)\notag \\
    & \leq \EE_{\substack{\cT^{t-1} \sim \pt(\cdot|\cH_{t-1}, s_t),\\ a\sim \algtheta(\cdot|\cT^{t-1}, s_t)} }r(s_t,a) + \epre\sqrt{1+\chi^2(\pt(\cdot|\cH_{t-1}, s_t), \decoder_{t-1}(\cdot|s_t))}\notag \\
    & \leq \EE_{\substack{\cT^{t-1} \sim \ptstar(\cdot|\cH_{t-1}, s_t),\\ a\sim \algtheta(\cdot|\cT^{t-1}, s_t)} }r(s_t,a) + \epre\sqrt{1+\chi^2(\ptstar(\cdot|\cH_{t-1}, s_t), \decoder_{t-1}(\cdot|s_t))}\notag \\
    & \leq \EE_{\substack{\cT^{t-1} \sim \ptstar(\cdot|\cH_{t-1}, s_t),\\ a\sim \overline{\alge}(\cdot|\cT^{t-1}, s_t)} }r(s_t,a) + 2\epre\sqrt{1+\chi^2(\ptstar(\cdot|\cH_{t-1}, s_t), \decoder_{t-1}(\cdot|s_t))}\notag \\
    & \leq \EE_{\substack{\cT^{t-1} \sim \ptstar(\cdot|\cH_{t-1}, s_t),\\ a\sim \overline{\alge}(\cdot|\cT^{t-1}, s_t)} }r(s_t,a) + 2C_{\decoder}\epre\sqrt{1+\chi^2(\ptstar(\cdot|\cH_{t-1}, s_t), \PP^{\Mb_{t}, \algc}_{t-1}(\cdot))},\label{online:2}
\end{align}
where the first inequality holds due to \eqref{online:1}, the second one holds due to the optimism principle, the third one holds due to \eqref{online:1}, and the last one holds due to Assumption \ref{ass:decoder}. Taking summation of \eqref{online:2} from $1$ to $T$ concludes our proof. 
\end{proof}

\section{More Experiment Details} \label{exp:details}

In this section, we provide additional details on our experiments in Section \ref{knnresults}. The tables included below outline the hyperparameter settings that were used throughout the evaluation process.

\begin{table}[H]
\centering
\caption{Hyperparameter settings ($k$ and $c$) for different benchmarks and model sizes.}
\begin{tabular}{c|c|c|c}
\toprule
\textbf{Benchmark} & \textbf{Alfworld} & \textbf{Webshop} & \textbf{HotPotQA} \\
\midrule
\textbf{Llama3.1 8b}  & $k=5,c=5$ & $k=5,c=0$ & $k=5,c=5$ \\
\textbf{Llama3.1 70b} & $k=5,c=5$ & $k=5,c=0$ & $k=5,c=0$ \\
\bottomrule
\end{tabular}

\label{tab:benchmark_hyperparams_updated}
\end{table}

Notably, there are a lot of burgeoning representative benchmarks and tasks across various domains, like reasoning \citep{DBLP:journals/corr/abs-2406-13975}, code generation \citep{DBLP:journals/corr/abs-2405-11430}, long context information extraction \citep{ge2024training, wang2023idea}, visual recognition \citep{wang2023parameter, kang20216, yang2022retargeting, chen2024gui}, math reasoning \citep{xiong2022self, zhao2024large}, classification \citep{wan2023text, kang2022tie, huang2024evolver, liu2023deid, tao2024harnessingllmsapiinteractions}, harm detection \citep{zhao2024biastoxicityroleplayreasoning, huang2023harnessing, fan2024towards}. These vast domain of applications holds the potential of \name{}. We leave to apply $\name$ to above benchmarks to future work.





\section{Prompt Template} \label{prompt}

Our prompt framework is designed with the considerations of simplicity, efficiency, and generalizability. Moreover, we aim to leverage the high-performing long-context capabilities of modern LLMs \citep{dubey2024llama} to a maximal degree. In one sentence to describe our prompting philosophy: previous success trajectories are seen by the agent as \emph{experiences} it has gone through in the same trial. As an example from the Alfworld experiment, a in-context demonstration is insert into the prefix of the agent's prompt as follows:
\newline
\begin{mdframed}
\begin{Verbatim}[breaklines=true,breaksymbol={},commandchars=\\\{\}]
[User] Interact with a household to solve a task. Here are a few examples.
You are in the middle of a room. Looking quickly around you, you see a cabinet 13, a cabinet 12, a cabinet 11, a cabinet 10, a cabinet 9, a cabinet 8, a cabinet 7, a cabinet 6, a cabinet 5, a cabinet 4, a cabinet 3, a cabinet 2, a cabinet 1, a coffeemachine 1, a countertop 1, a diningtable 1, a drawer 1, a fridge 1, a garbagecan 1, a microwave 1, a shelf 3, a shelf 2, a shelf 1, a sinkbasin 1, a stoveburner 4, a stoveburner 3, a stoveburner 2, a stoveburner 1, and a toaster 1.
Your task is to: put a hot apple in fridge.

[Assistant] think: To solve the task, I need to find and take an apple, then heat it with microwave, then put it in fridge.

[User] OK.

[Assistant] think: First I need to find an apple. An apple is more likely to appear in fridge (1), diningtable (1), coffeetable (1), drawer (1), cabinet (1-13), garbagecan (1). I can check one by one, starting with fridge 1.

[User] OK.

[Assistant] go to fridge 1

[User] The fridge 1 is closed.

\textbf{[experience detail omitted]}

[Assistant] go to fridge 1

[User] The fridge 1 is open. In it, you see a cup 1, and a egg 1.

[Assistant] put apple 1 in/on fridge 1

[User] You put the apple 1 in/on the fridge 1.

\textbf{[more in-context experiences]}

[User] Examples finished. Here is the task:
You are in the middle of a room. Looking quickly around you, you see a cabinet 6, a cabinet 5, a cabinet 4, a cabinet 3, a cabinet 2, a cabinet 1, a coffeemachine 1, a countertop 3, a countertop 2, a countertop 1, a drawer 3, a drawer 2, a drawer 1, a fridge 1, a garbagecan 1, a microwave 1, a shelf 3, a shelf 2, a shelf 1, a sinkbasin 1, a stoveburner 4, a stoveburner 3, a stoveburner 2, a stoveburner 1, and a toaster 1.
Your task is to: put a saltshaker in drawer.

[Assistant] think: To solve the task, I need to find and take a saltshaker, then put it in/on a drawer.

\textbf{[real trial begins]}
\end{Verbatim}
\end{mdframed}

\bibliographystyle{ims}
\bibliography{reference}

\end{document}